\documentclass{ieeetmlcn}
\usepackage{cite}
\usepackage{amsmath,amssymb,amsfonts,bm}
\usepackage{graphicx,color}
\usepackage{textcomp}
\usepackage{xcolor}
\usepackage{subcaption}
\usepackage{booktabs}
\usepackage{multirow}
\usepackage{algorithm,algorithmic}
\usepackage{hyperref}
\hypersetup{hidelinks=true}
\def\BibTeX{{\rm B\kern-.05em{\sc i\kern-.025em b}\kern-.08em
    T\kern-.1667em\lower.7ex\hbox{E}\kern-.125emX}}
\AtBeginDocument{\definecolor{tmlcncolor}{cmyk}{0.93,0.59,0.15,0.02}\definecolor{NavyBlue}{RGB}{0,86,125}}



\usepackage{epstopdf}
\def\OJlogo{\vspace{-4pt}\includegraphics[height=18pt]{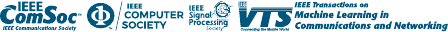}}
\def\seclogo{\vspace{10pt}\includegraphics[height=18pt]{new_logo}}
\def\authorrefmark#1{\ensuremath{^{\textbf{#1}}}}

\newtheorem{theorem}{Theorem}
\newtheorem{proposition}{Proposition}
\newtheorem{lemma}{Lemma}

\newtheorem{remark}{Remark}

\newtheorem{assumption}{Assumption}

\begin{document}
\receiveddate{XX Month, XXXX}
\reviseddate{XX Month, XXXX}
\accepteddate{XX Month, XXXX}
\publisheddate{XX Month, XXXX}
\currentdate{XX Month, XXXX}
\doiinfo{TMLCN.2026.0000000}

\markboth{}{Jeong {et al.}}

\title{Lagrangian or Lyapunov? Multi-Agent Reinforcement Learning under Coupled Finite-Horizon Constraints}

\author{Yeonseo Jeong\authorrefmark{1}, Graduate Student Member, IEEE,
Wonhyeok Ko\authorrefmark{2}, Graduate Student Member, IEEE,
and Songnam Hong\authorrefmark{1}, Senior Member, IEEE}
\affil{Department of Electronic Engineering, Hanyang University, Seoul 04763, South Korea}
\affil{Department of AI-Semiconductor, Hanyang University, Seoul 04763, South Korea}
\corresp{Corresponding author: Songnam Hong (email: snhong@hanyang.ac.kr).}
\authornote{Y. Jeong and W. Ko contributed equally to this work. This work was supported in part by the National Research Foundation of Korea(NRF) grant funded by the Korea government(MSIT) (RS-2024-00409492) and in part by the Institute of Information \& communications Technology Planning \& Evaluation (IITP) under the artificial intelligence semiconductor support program to nurture the best talents (IITP-2026-RS-2023-00253914) grant funded by the Korea government(MSIT).}

\begin{abstract}
Radio resource management in dense networks is subject to budgets that must hold on every realized trajectory, carried by different agent groups: an energy budget at each base station~(BS) and a handover budget at each user, coupled through the association decisions that consume both. Constrained multi-agent reinforcement learning~(MARL) handles such budgets using Lagrangian multipliers updated across episodes or Lyapunov virtual queues updated per slot, but their relative behavior under coupled two-sided constraints remains unclear. We formulate this problem and introduce a windowed proportional-fairness objective that captures the temporal distribution of service beyond the classical horizon-average formulation. We characterize the distinct budget-regulation timescales of the two mechanisms through complementary theoretical bounds and identify queue-price buildup as a key mechanism governing finite-horizon budget pacing. Using a common heterogeneous MARL architecture, we compare three utility surrogates and both constraint-handling mechanisms. The results show that utility design primarily determines throughput and fairness, whereas constraint-price dynamics govern training efficiency and budget pacing. Under the default setting, the queue-based PF policy reaches near-final performance after a single training episode, but its unscaled penalty can lead to delayed pacing under tight budgets. Guided by the analysis, we show that queue-price calibration reduces the empirical drift constant by an order of magnitude and substantially narrows the performance gap to Lagrangian methods, which adapt constraint prices across episodes at higher training cost. These findings provide practical design guidelines for constrained MARL.
\end{abstract}

\maketitle

\section{Introduction}
\label{sec:introduction}

The densification of wireless networks places cellular deployments under increasingly stringent energy and mobility constraints~\cite{buzzi2016survey}. In such networks, user association, base station~(BS) energy control, and mobility-induced handovers are inherently coupled, giving rise to fundamental tradeoffs~\cite{chen2011fundamental, ye2013user}. Keeping more BSs active can improve spatial diversity and throughput but drains the energy budget more quickly, whereas conservative activation preserves energy at the expense of coverage. Handovers add a third dimension: switching to a better-positioned BS can improve instantaneous rates, but each handover incurs signaling overhead and service interruption that degrade Quality of Experience~(QoE), so the number of handovers must itself be treated as a budgeted resource. Balancing these competing objectives under finite-horizon energy and handover budgets is the central challenge addressed in this work.

Jointly optimizing user association, BS activation, and handover control is a challenging sequential decision-making problem due to time-varying channels, inter-cell interference, and coupled finite-horizon constraints. Multi-agent reinforcement learning~(MARL) provides a framework for such distributed control by learning coordinated policies under the centralized training with decentralized execution~(CTDE) paradigm, while accounting for the long-term consequences of sequential decisions. MARL has been applied to dynamic power control in multi-cell networks~\cite{nasir2019multi}, joint spectrum and power optimization~\cite{zhang2025multi}, and distributed resource management in UAV-assisted vehicular networks~\cite{peng2021multi}. It has likewise been applied to user scheduling~\cite{yang2022marl}, distributed channel access~\cite{guo2022channel, hong2025qippo}, adaptive user association in mmWave networks~\cite{sa20}, resource management under interference constraints~\cite{naderializadeh2021resource}, handover management jointly with power allocation~\cite{guo2020joint}, and BS activation under a finite-horizon energy budget~\cite{ko2026lymarl}.

Applying MARL to this problem involves two key design choices: the utility surrogate and the constraint-handling mechanism. The proportional-fairness utility depends on time-averaged rates~\cite{kelly1997, kushner2004} and is not directly additive across slots, motivating the use of a per-slot surrogate reward. We consider three commonly used forms: the instantaneous log-rate surrogate motivated by Jensen's inequality~\cite{nasir2019multi, naderializadeh2021resource}, the rate normalized by its running average as in proportional-fair scheduling~\cite{kushner2004}, and the rate weighted by a virtual fairness queue that tracks service deficits through drift-plus-penalty~(DPP) optimization~\cite{neely2010stochastic, ko2026lymarl}. These surrogates induce different throughput--fairness tradeoffs. For constraint handling, we compare Lagrangian relaxation and Lyapunov virtual queues. In our finite-horizon formulation, Lagrangian multipliers are updated from episode-level constraint violations, building on constrained policy optimization~\cite{achiam2017cpo, gu2022macpo, aygun2026constrained}, whereas Lyapunov virtual queues update budget pressure at each slot~\cite{ko2026lymarl}. Although both approaches are established, their interaction with utility-surrogate design under coupled BS energy and user handover budgets remains insufficiently understood.

\begin{figure}[t!]
\centering
\includegraphics[width=1.0\linewidth]{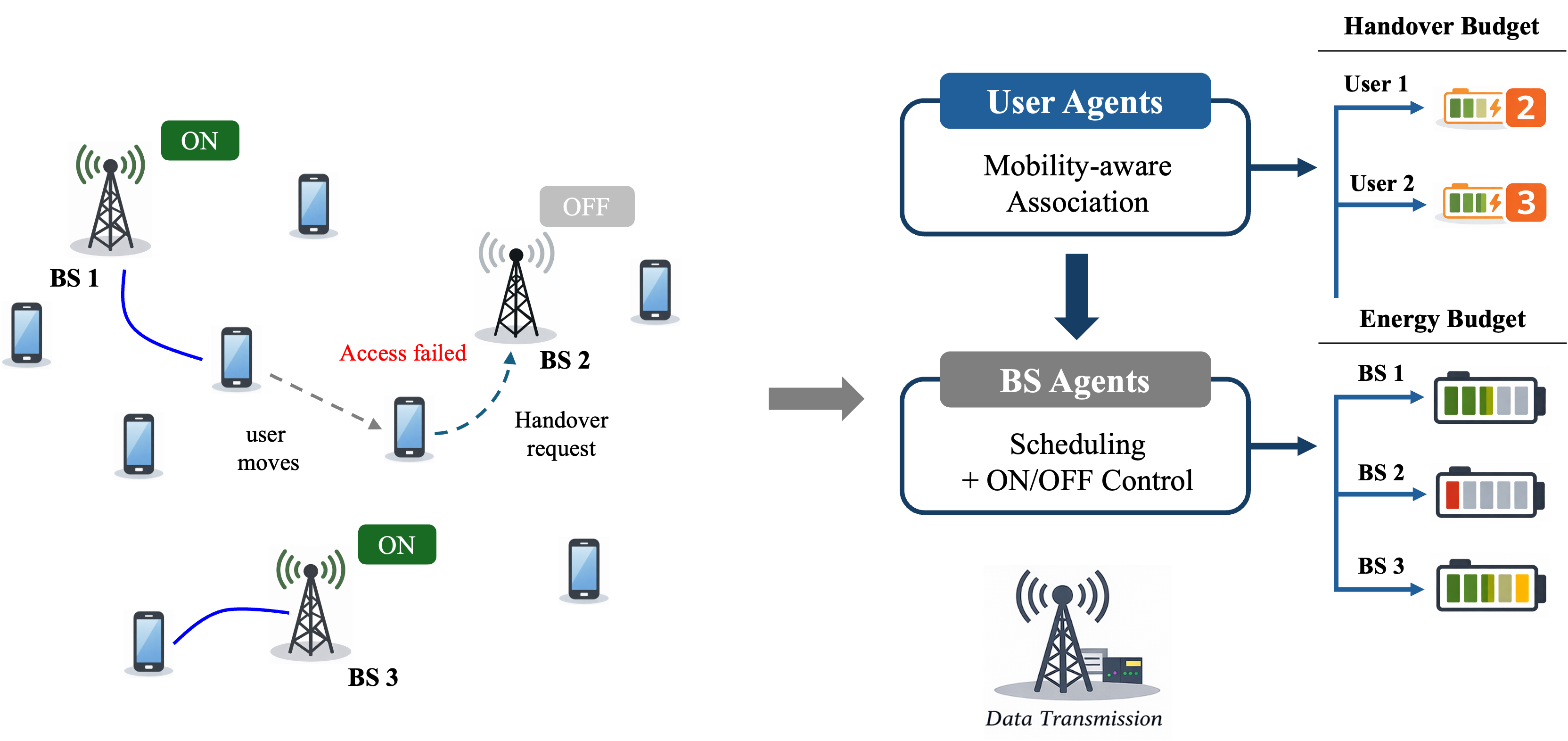}
\caption{Network scenario and common MARL architecture for joint user association, BS scheduling, and activation under coupled energy and handover budgets.}
\label{fig:1}
\end{figure}

The coupled two-sided setting makes both design choices consequential. Energy budgets belong to BSs, handover budgets to users, and a realized handover depends jointly on association requests and BS scheduling. Unlike the one-sided energy-budget framework of~\cite{ko2026lymarl}, we model both costs in a shared reward. Hard budgets are enforced on each realized trajectory through inference-time action masking, separately from the learned constraint prices.

To address this gap, we compare three utility surrogates and two constraint-handling mechanisms within a common heterogeneous MARL architecture (Fig.~\ref{fig:1}), keeping the network model and learning backbone fixed. The theoretical and experimental findings lead to the following contributions:
\begin{itemize}
\item \textbf{Coupled finite-horizon formulation.} We formulate radio resource management under coupled BS energy and user handover budgets that hold on every trajectory, and introduce a windowed proportional-fairness~(PF) objective that, unlike the classical one, penalizes temporal concentration of service.
\item \textbf{Theoretical characterization.} We characterize the distinct budget-regulation timescales of episode-level Lagrangian multipliers and slot-level virtual queues, and establish complementary bounds on constraint violations and budget consumption. A price-buildup analysis further explains the delayed budget regulation of queue-based methods and predicts its dependence on the queue-price scale.
\item \textbf{Mechanistic insights and design principles.} We show that utility design primarily determines throughput and fairness, whereas constraint-price dynamics govern learning efficiency and budget pacing. Guided by the price-buildup analysis, we demonstrate that queue-price calibration substantially improves budget regulation under tight constraints while retaining the rapid per-run training of virtual queues, leading to practical design guidelines.
\end{itemize}

The remainder of this paper is organized as follows. Sections~\ref{sec:model}--\ref{sec:marl} present the problem formulation, utility surrogates, and common MARL framework. Section~\ref{sec:theory} provides the theoretical analysis, Section~\ref{sec:experiments} presents the experimental results and design guidelines, and Section~\ref{sec:conclusion} concludes the paper.

\section{System Model and Problem Formulation}
\label{sec:model}

We consider joint user association, scheduling, and BS activation subject to hard finite-horizon energy budgets at BSs and handover budgets at users.

\subsection{Network Model}
\label{sec:network_model}

We consider a cellular network with $B$ BSs, indexed by $\mathcal{B}=\{1,\dots,B\}$, and $U$ mobile users, indexed by $\mathcal{U}=\{1,\dots,U\}$, operating over discrete time slots $t=0,1,\ldots,T-1$. At each slot, user--BS association, user scheduling, and BS activation are jointly determined under time-varying channel conditions and inter-cell interference.

\noindent\textbf{Association and scheduling.}
Let $x_{u,b}(t)\in\{0,1\}$ indicate whether user $u$ is served by BS $b$ at slot $t$, and let $y_b(t)\in\{0,1\}$ indicate whether BS $b$ is active. At each slot, each user is served by at most one BS, and each active BS serves at most one user:
\begin{align}
  \sum_{b=1}^{B} x_{u,b}(t) &\leq 1, \quad \forall u,t,
  \label{eq:assoc} \\
  \sum_{u=1}^{U} x_{u,b}(t) &\leq y_b(t), \quad \forall b,t.
  \label{eq:sched}
\end{align}
The service rate of user $u$ at slot $t$ is $R_u(t)=\sum_{b=1}^{B} x_{u,b}(t)\,r_{u,b}(t)$, where $r_{u,b}(t)$ denotes the achievable rate, determined by the channel realization and the inter-cell interference induced by network-wide decisions. Control decisions are made on the basis of estimated rates $\{\hat{r}_{u,b}(t)\}$ obtained from available channel state information~(CSI). The single-user-per-BS restriction isolates the coupling between BS energy and user handover budgets from multi-user scheduling. Multi-user transmission would change the scheduling decisions and handover demand but not the definitions of these budgets under the adopted energy and handover models; extending the framework to multi-user transmission is left for future work.

\noindent\textbf{Handover model.}
 User mobility and BS load variations can change the preferred serving BS, and reassociation incurs signaling overhead and may interrupt service. Let $m_u(t)\in\{0,1,\ldots,B\}$ record the BS that most recently served user $u$ up to slot $t$, with $m_u(t)=0$ if the user has not yet been served by any BS. A handover occurs when a served user switches its serving BS:
\begin{equation}
  h_u(t) = \mathbf{1}\!\left\{
    \sum_{b=1}^{B} x_{u,b}(t) = 1,\;
    \substack{
        m_u(t-1)\neq 0,\\
        m_u(t-1)\neq b_u(t)
    }
  \right\},
  \label{eq:ho}
\end{equation}
where $b_u(t) = \arg\max_b x_{u,b}(t)$ denotes the serving BS when $\sum_b x_{u,b}(t)=1$. For later use, we also define the \emph{predicted} handover indicator $\hat{h}_{u,b}(t)=\mathbf{1}\{b\neq m_u(t-1),\,
m_u(t-1)\neq 0\}$, which flags whether serving user $u$ at BS $b$ \emph{would} trigger a handover. The memory state evolves as
\begin{equation}
  m_u(t) =
  \begin{cases}
    b_u(t),     & \text{if } \sum_b x_{u,b}(t) = 1, \\
    m_u(t-1),   & \text{otherwise,}
  \end{cases}
  \label{eq:memory}
\end{equation}
which retains the last serving BS when the user is not served. A handover occurs only when a user requests a different BS and that BS schedules the user, coupling the handover cost to the decisions of both agent groups.

\noindent\textbf{Energy and handover budgets.}
Dynamic BS activation can reduce energy consumption in dense networks~\cite{Oh11, Wu13}. The per-slot energy consumption of BS $b$ is $e_b(t)=\bar{e}_b y_b(t)$, where $\bar{e}_b>0$ is a fixed per-slot activation cost. The finite-horizon energy and handover budgets are defined as
\begin{equation}
  E_b^{\max} = \eta\bar{e}_b T, \quad
  H_u^{\max} = \left\lfloor \kappa (T-1)\right\rfloor,
  \label{eq:budgets}
\end{equation}
where $\eta\in(0,1]$ specifies the maximum allowable BS activation ratio and $\kappa\in[0,1]$ the maximum handover ratio per user. We consider a fixed per-slot BS activation cost and represent handover overhead by the number of serving-BS changes, without additional BS switching costs or minimum dwell-time constraints. Both budgets are additive over time and can therefore be handled by the constraint-handling mechanisms in Section~\ref{sec:marl}.

\subsection{Problem Formulation}

Let $\mathcal{F}$ denote the feasibility set defined by~\eqref{eq:assoc}--\eqref{eq:sched}. We consider a decentralized joint policy $\pi=(\pi^{\rm user},\pi^{\rm bs})$, where each user selects an association request based on its local observation, and each BS determines scheduling and activation based on its observation and the received requests. These policies jointly induce the network-wide decisions $\{x_{u,b}(t),y_b(t)\}$ at each slot. We partition the horizon into $N_W=T/W$ windows of length $W$, assuming that $W$ divides $T$, and define $\bar{R}_u^{(k)}=\frac{1}{W}\sum_{t=kW}^{(k+1)W-1}R_u(t)$ as the average rate of user $u$ in window $k$. The objective is to maximize the windowed proportional-fairness~(PF) utility subject to hard finite-horizon budgets:
\begin{align}
  \max_{\pi} \quad
  & \mathbb{E}_{\pi}\!\left[\sum_{k=0}^{N_W-1}\sum_{u=1}^{U}
    \log\bigl(\bar{R}_u^{(k)}+\epsilon\bigr)\right]
  \label{eq:obj} \\
  \text{s.t.} \quad
  & \sum_{t=0}^{T-1} \bar{e}_b y_b(t)\leq E_b^{\max},\quad \forall b,\ \text{a.s.},
  \label{eq:energy_c} \\
  & \sum_{t=1}^{T-1} h_u(t)\leq H_u^{\max},\quad \forall u,\ \text{a.s.},
  \label{eq:ho_c} \\
  & \{x_{u,b}(t), y_b(t)\} \in \mathcal{F},\quad \forall t.
  \label{eq:feas_c}
\end{align}
Here, $\epsilon>0$ ensures that the logarithm is well-defined, $\mathbb{E}_{\pi}[\cdot]$ denotes the expectation over channel variations, user mobility, and policy stochasticity, and a.s.\ denotes almost-sure satisfaction with respect to the same randomness. The handover count begins at $t=1$, with $m_u(-1)=0$ so that the first association of a user is not counted as a handover.

Two aspects of~\eqref{eq:obj}--\eqref{eq:feas_c} distinguish the formulation from conventional constrained MARL. First, when $W=T$, the objective reduces to the classical PF utility of horizon-average rates~\cite{kelly1997,kushner2004}, which is insensitive to the temporal distribution of a fixed total service. For example, a policy may deliver a user's service entirely in the first $60\%$ of the horizon and none thereafter while achieving the same PF utility as a policy that distributes the same service uniformly. The proposed windowed objective favors more uniform service across windows through the concavity of the logarithm: an unserved window contributes $\log\epsilon$, with $\epsilon=10^{-3}$~Gbps in our experiments. It therefore discourages prolonged service gaps without imposing a separate continuity constraint. Second, the energy and handover budgets must be satisfied on every realized trajectory, unlike expected-cost constraints in conventional constrained MARL~\cite{achiam2017cpo,gu2022macpo}, which may permit violations on individual trajectories.

These features require per-slot surrogates for the nonadditive PF objective and budget-aware control of coupled association, scheduling, and activation decisions. Sections~\ref{sec:reward} and~\ref{sec:marl} address these aspects, respectively, with the handover cost determined using the last-serving-BS state in~\eqref{eq:memory}.

\section{Per-Slot Utility Surrogates}
\label{sec:reward}

The windowed PF objective in~\eqref{eq:obj} is not additive across slots. We therefore consider three per-slot utility surrogates for MARL: an instantaneous log-rate reward motivated by Jensen's inequality, a marginal-utility weight based on proportional-fair scheduling, and a service-deficit weight from drift-plus-penalty~(DPP) optimization. These surrogates represent only the utility component; Section~\ref{sec:marl} incorporates the energy and handover constraints through Lagrangian multipliers or Lyapunov virtual queues.

\subsection{Jensen Surrogate}
\label{subsec:jensen}

Applying Jensen's inequality to the concave logarithmic utility within each window yields the following lower bound for every realized trajectory:
\begin{equation}
  \sum_{k=0}^{N_W-1}\sum_{u=1}^{U}\log\bigl(\bar{R}_u^{(k)}+\epsilon\bigr)
  \geq\frac{1}{W}\sum_{t=0}^{T-1}\sum_{u=1}^{U} \log\bigl(R_u(t)+\epsilon\bigr).
  \label{eq:jensen}
\end{equation}
Since the factor $1/W$ does not affect the maximizer of the lower bound for a fixed $W$, we obtain the per-slot reward
\begin{equation}
    r_t=\sum_{u=1}^{U}\log\bigl(R_u(t)+\epsilon\bigr).
    \label{eq:jensen_reward}
\end{equation}
The Jensen bound can be loose when user rates vary substantially across slots. Unlike the windowed PF objective, this surrogate depends only on instantaneous rates and does not account for previously delivered service. In particular, an unserved user contributes $\log\epsilon$ at every slot, resulting in a substantial penalty under sparse scheduling. The energy and handover constraints remain those in~\eqref{eq:energy_c}--\eqref{eq:feas_c} and are handled separately in Section~\ref{sec:marl}.

\subsection{PF Surrogate}
\label{subsec:pf}

For a realized trajectory, the windowed PF objective in~\eqref{eq:obj} has marginal utility $1/[W(\bar{R}_u^{(k)}+\epsilon)]$ with respect to $R_u(t)$ for slot $t$ in window $k$. Motivated by this marginal-utility structure and proportional-fair scheduling~\cite{kelly1997,kushner2004}, we define the per-slot reward
\begin{equation}
  r_t = \sum_{u=1}^{U}\frac{R_u(t)}{\max\{\widehat{R}_u(t-1),\delta_{\rm PF}\}},
  \qquad kW\le t<(k+1)W,
  \label{eq:pf_reward}
\end{equation}
where $\widehat{R}_u(t-1)$ is the running rate estimate
\begin{equation}
  \widehat{R}_u(t-1)=\frac{nR_{u,\mathrm{init}}^{(k)}+\sum_{\tau=kW}^{t-1}R_u(\tau)}{n+t-kW}.
  \label{eq:pf_avg}
\end{equation}
Here, $n>0$ is the number of pseudo-samples, $R_0>0$ is the initial rate estimate, and $\delta_{\rm PF}>0$ is a lower bound on the reward denominator. The window initialization is $R_{u,\mathrm{init}}^{(0)}=R_0$ and $R_{u,\mathrm{init}}^{(k)}=(nR_{u,\mathrm{init}}^{(k-1)}+W\bar{R}_u^{(k-1)})/(n+W)$ for $k\ge1$. This initialization retains a smoothed estimate of past service while emphasizing observations within the current window. Consequently, smaller windows increase sensitivity to recent service deficits. With $W=T$, \eqref{eq:pf_avg} reduces to a single running average over the horizon, as used by PF-Lag and PF-Queue. Their windowed variants, PF-Lag-W and PF-Queue-W, use $W<T$ during training. As with the Jensen surrogate, the energy and handover constraints are incorporated separately in Section~\ref{sec:marl}.

\subsection{DPP Surrogate}
\label{subsec:dpp}

The drift-plus-penalty~(DPP) framework~\cite{neely2010stochastic} uses virtual queues to transform long-term utility optimization into per-slot decisions. Following~\cite{ko2026lymarl}, we introduce an auxiliary target rate $\gamma_u(t)=\min\{\gamma_u^{\max}(t),V/\max\{Q_u(t),\delta_Q\}\}$, where $\gamma_u^{\max}(t)$ is a reference-rate cap based on candidate-BS rate estimates under the previous-slot interference state, $V>0$ controls the utility--fairness tradeoff, and $\delta_Q>0$ prevents division by zero. The fairness queue evolves over the full horizon as
\begin{equation}
  Q_u(t+1)=\bigl[Q_u(t)+\gamma_u(t)-R_u(t)\bigr]^+,
  \label{eq:Q}
\end{equation}
with $Q_u(0)=q_0>0$. Unlike the window-local PF rate estimate, this queue retains accumulated service deficits throughout the episode.

Applying the one-step DPP analysis to the quadratic Lyapunov function $\frac{1}{2}\sum_u Q_u(t)^2$ gives the following service-dependent term, which we adopt as the per-slot utility surrogate:
\begin{equation}
  r_t = \sum_{u=1}^{U} Q_u(t)R_u(t).
  \label{eq:dpp_util}
\end{equation}
Unlike the Jensen and PF surrogates, the DPP reward weights each user's rate by its accumulated service deficit. This provides a time-varying fairness signal that reflects the user's service history.

Neither the utility surrogates nor the price-based constraint penalties alone guarantee the per-trajectory budgets in~\eqref{eq:energy_c}--\eqref{eq:ho_c}. Budget-aware action masking therefore ensures feasibility at inference for all compared methods~(Remark~\ref{rem:masking}), while the learned policies and constraint-handling mechanisms determine how the available budgets are allocated over time.

\section{Two Constraint-Handling Families on One Architecture}
\label{sec:marl}

We develop a common heterogeneous MARL architecture under the CTDE paradigm to compare the three utility surrogates and two constraint-handling mechanisms. The network is modeled as a finite-horizon decentralized partially observable Markov decision process~(Dec-POMDP), with user and BS agents trained using HAPPO~\cite{kuba2021trust}. All methods share the agent design, candidate-set construction, and budget-aware action masking at inference, as described in Section~\ref{subsec:agent_design}. Their per-slot rewards take the common form
\begin{equation}
  \tilde{r}_t = r_t
  - \sum_{b=1}^{B}\lambda_b^E(t)\,\bar{e}_b\,y_b(t)
  - \sum_{u=1}^{U}\lambda_u^H(t)\,h_u(t),
  \label{eq:reward_form}
\end{equation}
where $r_t$ is one of the three utility surrogates in Section~\ref{sec:reward}, and $\lambda_b^E(t)$ and $\lambda_u^H(t)$ are the energy and handover constraint prices. The Lagrangian family uses episode-level multipliers $(\mu_E^b,\mu_H^u)$, whereas the Lyapunov family uses slot-level prices derived from virtual queues $(Z_b(t),G_u(t))$. The two families also differ in their budget-related observations. Section~\ref{subsec:grid} summarizes the resulting design grid and its windowed variants.

\subsection{Agent Design and HAPPO Backbone}
\label{subsec:agent_design}

Two heterogeneous agent groups act at each slot. Each user selects a BS association request $a_u^{\rm user}(t)\in\{1,\ldots,B\}$. A request to a BS different from $m_u(t-1)$ constitutes a handover attempt only when $m_u(t-1)\neq0$. Each BS selects a scheduling and activation action $a_b^{\rm bs}(t)\in\{0,1,\ldots,N_c\}$, where $a_b^{\rm bs}(t)=0$ denotes inactivity and $a_b^{\rm bs}(t)=k$ schedules the $k$-th user in the candidate set $\mathcal{C}_b(t)$.

Let $\mathcal{U}_b^{\rm req}(t)$ denote the users requesting BS $b$. To maintain fixed-dimensional BS action and observation spaces, each BS retains up to $N_c$ candidates in $\mathcal{C}_b(t)\subseteq\mathcal{U}_b^{\rm req}(t)$, with empty positions zero-padded and masked. Candidates are prioritized using the surrogate-dependent score
\begin{equation}
  s_{u,b}(t)\;\triangleq\;w_{u}(t)\,\hat{r}_{u,b}(t)
  -\lambda_u^H(t)\,\hat{h}_{u,b}(t).
  \label{eq:score}
\end{equation}
Here, $w_u(t)$ is the rate weight used for candidate ranking: $Q_u(t)$ for DPP, $1/\max\{\widehat{R}_u(t-1),\delta_{\rm PF}\}$ for PF, and $1$ for Jensen. The Jensen score uses an instantaneous-rate heuristic, since its logarithmic reward is not linear in rate. The handover price $\lambda_u^H(t)$ enters the candidate score, whereas the energy penalty is identical for all candidates of an active BS and does not affect their ranking. The joint actions determine $\{x_{u,b}(t),y_b(t)\}\in\mathcal{F}$, satisfying~\eqref{eq:feas_c}.

All user agents share a common policy $\pi_{\theta^{\rm user}}$ and all BS agents share $\pi_{\theta^{\rm bs}}$. The two group policies are updated sequentially via the clipped HAPPO objective
\begin{equation}
  L_{\rm HAPPO}(\theta^g)
  = \mathbb{E}\!\left[
      \min\!\left(
        \rho_t^g A_t^g,\;
        {\rm clip}(\rho_t^g,
        1{-}\varepsilon_{\rm clip},
        1{+}\varepsilon_{\rm clip}) A_t^g
      \right)
    \right],
  \label{eq:happo_loss}
\end{equation}
where $\rho_t^g = \pi_{\theta^g}(a_t \mid o_t)/\pi_{\theta^g_{\rm old}}(a_t \mid o_t)$ is the importance sampling ratio for an individual agent in group $g \in \{{\rm user, bs}\}$. The advantage $A_t$ is computed from the centralized critic $V_\phi(s_t)$ via generalized advantage estimation~(GAE),
\begin{align}
  \delta_t &= \tilde{r}_t + \gamma(1-d_t) V_\phi(s_{t+1}) - V_\phi(s_t),
  \label{eq:td_error} \\
  A_t &= \delta_t + \gamma\lambda(1-d_t) A_{t+1},
  \label{eq:gae}
\end{align}
where $d_t = \mathbf{1}\{t = T-1\}$ is the episode termination indicator, $\gamma \in (0,1)$ is the discount factor, and $\lambda \in (0,1)$ controls the bias--variance tradeoff. The user policy is updated first with $A_t^{\rm user}=A_t$. Its effect on the joint action distribution is then carried into the BS-policy update through the correction factor
\begin{equation}
C_t^{\rm user} =
    \prod_{u=1}^{U}
    \frac{\pi_{\theta_{\rm new}^{\rm user}}
    \left(a_u(t)\mid o_u(t)\right)}
    {\pi_{\theta_{\rm old}^{\rm user}}
    \left(a_u(t)\mid o_u(t)\right)},
\label{eq:user_correction}
\end{equation}
which accounts for the change in the joint user-policy distribution. The BS policy is then updated with $A_t^{\rm bs}=C_t^{\rm user}A_t$, incorporating the preceding user-policy update. Because agents within each group share parameters, we do not invoke the per-agent monotonic-improvement guarantee of~\cite{kuba2021trust}. The actor policies are updated every $L$ slots.

\begin{remark}[Budget-Aware Action Masking]\label{rem:masking}
At inference, hard budgets are enforced through action masking based on the remaining energy $\tilde{E}_b(t)=E_b^{\max}-\sum_{\tau<t}\bar{e}_b y_b(\tau)$ and handover allowance $\tilde{H}_u(t)=H_u^{\max}-\sum_{\tau<t}h_u(\tau)$. BS $b$ is forced inactive when $\tilde{E}_b(t)<\bar{e}_b$. When $\tilde{H}_u(t)=0$ and $m_u(t-1)\neq0$, user $u$ may request only its last serving BS; a user not previously served may request any BS without incurring a handover. Budget-aware masking is disabled during training so that constraint violations provide learning signals. The same inference masking is applied to all methods and guarantees per-trajectory feasibility~(Proposition~\ref{prop:feasibility}).
\end{remark}

\subsection{Constraint Handling by Lagrangian Multipliers}
\label{subsec:constrained_happo}

The multiplier family applies Lagrangian relaxation to the expected-cost counterparts of the energy and handover constraints~\eqref{eq:energy_c}--\eqref{eq:ho_c}~\cite{gu2022macpo}. A nonnegative multiplier is assigned to each BS energy budget and user handover budget, yielding the constraint prices in~\eqref{eq:reward_form},
\begin{equation}
  \lambda_b^E(t)=\mu_E^b\ge0,\qquad \lambda_u^H(t)=\mu_H^u\ge0,
  \label{eq:aug_reward}
\end{equation}
which remain constant within each episode. Initialized at $\mu_E^b(1)=\mu_H^u(1)=0$, the multipliers are updated from the realized constraint violations at the end of each training episode $k$:
\begin{align}
  \mu_E^b(k{+}1) &= \biggl[\mu_E^b(k) + \frac{\beta_E}{\sqrt{K}} C_E^b(k)\biggr]^+,
  \label{eq:dual_E} \\
  \mu_H^u(k{+}1) &= \biggl[\mu_H^u(k) + \frac{\beta_H}{\sqrt{K}} C_H^u(k)\biggr]^+,
  \label{eq:dual_H}
\end{align}
with step sizes $\beta_E/\sqrt{K}$ and $\beta_H/\sqrt{K}$ for the energy and handover multipliers, respectively, where $K$ is the total number of training episodes and
\begin{align}
  C_E^b(k) &= \tfrac{1}{T}\sum_{t=0}^{T-1} \bar{e}_b y_b(t) - \bar{E}_b,
  \label{eq:cost_E} \\
  C_H^u(k) &= \tfrac{1}{T}\sum_{t=1}^{T-1} h_u(t) - \bar{H}_u,
  \label{eq:cost_H}
\end{align}
are the episode-level constraint violations. Constraint awareness within the episode is provided through the normalized remaining budgets
\begin{align}
  \mathrm{Rem}_b^E(t)
  &= \frac{E_b^{\max} - \sum_{\tau=0}^{t-1}\bar{e}_b y_b(\tau)}{E_b^{\max}},
  \label{eq:rem_e} \\
  \mathrm{Rem}_u^H(t)
  &= \frac{H_u^{\max} - \sum_{\tau=1}^{t-1} h_u(\tau)}{H_u^{\max}},
  \label{eq:rem_h}
\end{align}
which enter the state and observations: the centralized critic observes $s(t)=\bigl(\{\hat{r}_{u,b}(t)\}_{u,b},\,\{\mathrm{Rem}_b^E(t)\}_b,\,\{\mathrm{Rem}_u^H(t), m_u(t{-}1)\}_u\bigr)$, and each user and BS observe the locally available subsets
\begin{align*}
  o_u^{\rm user}(t) &= \Bigl(\{\hat{r}_{u,b}(t),\,\mathrm{Rem}_b^E(t)\}_b,\,
    \mathrm{Rem}_u^H(t),\, m_u(t{-}1)\Bigr),\\
  o_b^{\rm bs}(t) &= \Bigl(\mathrm{Rem}_b^E(t),\,
    \{\hat{r}_{u,b}(t),\,\mathrm{Rem}_u^H(t)\}_{u \in\mathcal{C}_b(t)}\Bigr),
\end{align*}
with $Q_u(t)$ added when the DPP surrogate is used.

\subsection{Constraint Handling by Virtual Queues}
\label{subsec:unconstrained_happo}

The queue family tracks energy and handover budget consumption through slot-level virtual queues:
\begin{align}
  Z_b(t+1) &= \bigl[Z_b(t) + \bar{e}_b y_b(t) - \bar{E}_b\bigr]^+, \quad \forall b,
  \label{eq:Z}\\
  G_u(t+1) &= \bigl[G_u(t) + h_u(t) - \bar{H}_u\bigr]^+, \quad \forall u,
  \label{eq:G}
\end{align}
where $Z_b(0)=G_u(0)=0$, $\bar{E}_b=E_b^{\max}/T$, and $\bar{H}_u=H_u^{\max}/T$. These queues accumulate excess consumption relative to the per-slot allowances, with projection onto nonnegative values. While the energy queue depends directly on BS activation, the handover queue depends on the joint association and scheduling decisions. The quadratic Lyapunov drift bound yields the queue-dependent constraint prices
\begin{equation}
  \lambda_b^E(t)=Z_b(t),\qquad \lambda_u^H(t)=G_u(t),
  \label{eq:dpp_reward}
\end{equation}
for the unscaled formulation. The price-scaling variants in Section~\ref{subsec:robust} use $\lambda_b^E(t)=\alpha_E Z_b(t)$. Unlike Lagrangian multipliers, these prices evolve within each episode. During training, the policies maximize $\mathbb{E}_{\pi}[\sum_{t=0}^{T-1}\tilde{r}_t]$, with the budget penalties incorporated into the reward.

The budget queues replace the remaining-budget features in the state and observations. The centralized critic observes $s(t)=\bigl(\{G_u(t),m_u(t{-}1)\}_u,\,\{Z_b(t)\}_b,\,\{\hat{r}_{u,b}(t)\}_{u,b}\bigr)$, and the decentralized observations are
\begin{align}
  o_u^{\rm user}(t) &= \Bigl(G_u(t),\, m_u(t{-}1),\,
    \{\hat{r}_{u,b}(t),\, Z_b(t)\}_b\Bigr),
  \label{eq:obs_user_c}\\
  o_b^{\rm bs}(t) &= \Bigl(Z_b(t),\,
    \{G_u(t),\, \hat{r}_{u,b}(t)\}_{u\in\mathcal{C}_b(t)}\Bigr),
  \label{eq:obs_bs_c}
\end{align}
with $Q_u(t)$ included when the DPP surrogate is used. Each BS broadcasts its queue value $Z_b(t)$ once per slot, requiring $B$ scalar queue values to be communicated to the users. All budget queues are reset to zero at episode boundaries. The queue states affect policy learning through both the reward penalties in~\eqref{eq:reward_form} and the observations, allowing the agents to condition their actions on evolving budget pressure. Algorithm~\ref{alg:helymarl} summarizes the training procedure; the multiplier family follows the same procedure with episode-level dual updates and remaining-budget observations.

\begin{table*}[t]
\centering
\caption{Design grid of three utility surrogates and two constraint-handling families. All methods use~\eqref{eq:reward_form}; $w_u$ denotes the candidate-ranking weight in~\eqref{eq:score}, and \textit{-W} denotes the windowed PF variants.}
\label{tab:grid}
\renewcommand{\arraystretch}{1.3}
\setlength{\tabcolsep}{7pt}
\footnotesize
\begin{tabular}{l l c | c c}
\toprule
\multicolumn{3}{c|}{\textbf{Utility surrogate (Section~\ref{sec:reward})}}
& \multicolumn{2}{c}{\textbf{Constraint handling (Section~\ref{sec:marl})}} \\
\cmidrule(lr){1-3}\cmidrule(lr){4-5}
& & & \textbf{Lagrangian multipliers} & \textbf{Virtual queues} \\
\textbf{Name} & \textbf{Reward term } $r_t$ & \textbf{Weight } $w_u$
& prices $(\mu_E^b,\mu_H^u)$~\eqref{eq:aug_reward} & prices $(Z_b,G_u)$~\eqref{eq:dpp_reward} \\
& & & updated per episode & updated per slot \\
& & & obs.\ $\mathrm{Rem}^E_b,\mathrm{Rem}^H_u$ & obs.\ $Z_b,G_u$ \\
\midrule
Jensen & $\sum_u\log(R_u+\epsilon)$~\eqref{eq:jensen_reward} & $1$ & Jensen-Lag & Jensen-Queue \\
PF & $\sum_u w_u R_u$~\eqref{eq:pf_reward} & $1/\max\{\widehat{R}_u(t-1),\delta_{\rm PF}\}$ & PF-Lag, PF-Lag-W & PF-Queue, PF-Queue-W \\
DPP & $\sum_u Q_uR_u$~\eqref{eq:dpp_util} & $Q_u$ & DPP-Lag & DPP-Queue \\
\bottomrule
\end{tabular}
\end{table*}

\begin{algorithm}[t]
\caption{Training Procedure for Virtual-Queue MARL}
\label{alg:helymarl}
\begin{algorithmic}[1]
\REQUIRE $K$, $T$, $L$, $N_c$; utility surrogate $r_t$; initial $\pi_{\theta^{\rm user}}$, $\pi_{\theta^{\rm bs}}$, $V_\phi$
\FOR{episode $k = 1, \ldots, K$}
    \STATE Reset $Z_b(0),G_u(0)\leftarrow0$ and the surrogate's state; $\mathcal{D} \leftarrow \emptyset$
    \FOR{$t = 0, \ldots, T-1$}
        \STATE Users observe~\eqref{eq:obs_user_c} and sample $a_u^{\rm user}(t)\sim\pi_{\theta^{\rm user}}$
        \STATE BSs form $\mathcal{C}_b(t)$ by~\eqref{eq:score}, observe~\eqref{eq:obs_bs_c}, and sample $a_b^{\rm bs}(t)\sim\pi_{\theta^{\rm bs}}$
        \STATE Execute actions; observe $\{R_u(t)\},\{h_u(t)\}$; compute $\tilde{r}_t$ by~\eqref{eq:reward_form}
        \STATE Update $Z_b,G_u$ by~\eqref{eq:Z}--\eqref{eq:G} and the surrogate's state; store transitions in $\mathcal{D}$
        \IF{$(t{+}1) \bmod L = 0$ \textbf{or} $t = T-1$}
            \STATE Compute $\{A_t\}$ by~\eqref{eq:td_error}--\eqref{eq:gae} and update $V_\phi$
            \STATE Update $\pi_{\theta^{\rm user}}$ with $A_t$ by~\eqref{eq:happo_loss}
            \STATE Compute $C_t^{\rm user}$ by~\eqref{eq:user_correction}
            \STATE Update $\pi_{\theta^{\rm bs}}$ with $C_t^{\rm user}A_t$ by~\eqref{eq:happo_loss}
            \STATE $\pi_{\rm old}\leftarrow\pi_\theta$; $\mathcal{D}\leftarrow\emptyset$
        \ENDIF
    \ENDFOR
\ENDFOR
\end{algorithmic}
\end{algorithm}

\subsection{Design Grid}
\label{subsec:grid}

Crossing the three utility surrogates with the two constraint-handling families yields six methods and two windowed PF variants, as summarized in Table~\ref{tab:grid}. The utility surrogate determines the reward $r_t$ in~\eqref{eq:reward_form}, the candidate-ranking weight $w_u$ in~\eqref{eq:score}, and the surrogate-related state variables, whereas the constraint-handling family determines the constraint prices and budget-related observations. The suffix \textit{-W} denotes a PF running average over windows with $W<T$, while \textit{-Queue} refers to the budget queues $Z_b$ and $G_u$, distinct from the fairness queue $Q_u$ used in both DPP methods. All methods share the HAPPO architecture and inference-time action masking, enabling controlled comparisons across the two design axes. DPP-Queue extends the energy-budget framework of~\cite{ko2026lymarl} to coupled BS energy and user handover budgets through a shared reward. The non-learning DDPP and MaxSNR baselines are introduced in Section~\ref{sec:experiments}.

\section{Theoretical Analysis}
\label{sec:theory}

We characterize the distinct timescales on which Lagrangian multipliers and Lyapunov virtual queues regulate finite-horizon budgets. Proposition~\ref{prop:feasibility} establishes per-trajectory feasibility through inference-time action masking, separating hard-budget enforcement from learned budget pacing. Under Assumption~\ref{ass:std} and the mean-dual dynamics, Theorem~\ref{thm:dual_conv} establishes an $\mathcal{O}(1/\sqrt{K})$ bound on episode-averaged expected constraint violations for the multiplier family. For virtual queues, Proposition~\ref{prop:intra_episode} bounds expected cumulative consumption over every partial horizon under a queue-weighted drift condition. Section~\ref{subsec:counterexample} analyzes how queue-price buildup and its scale affect finite-horizon pacing, and Section~\ref{subsec:joint} summarizes the resulting inter-episode and intra-episode tradeoffs.

\subsection{Per-Trajectory Feasibility}
\label{subsec:feasibility}

Neither constraint-handling mechanism alone guarantees~\eqref{eq:energy_c}--\eqref{eq:ho_c} on every trajectory: multipliers regulate expected constraint violations, whereas queue penalties discourage excess consumption. Per-trajectory feasibility is enforced through inference-time action masking.

\begin{proposition}[Per-Trajectory Feasibility]
\label{prop:feasibility}
Under the inference-time action masking of Remark~\ref{rem:masking}, for any trained policy $\pi_{\theta}$ and any environment realization,
\begin{equation*}
  \sum_{t=0}^{T-1}\bar{e}_b y_b(t)\le E_b^{\max},\ \forall b,
  \qquad
  \sum_{t=1}^{T-1} h_u(t)\le H_u^{\max},\ \forall u.
\end{equation*}
\end{proposition}

\begin{IEEEproof}
Let $\tilde{E}_{b}(t)=E_{b}^{\max}-\sum_{\tau<t}\bar{e}_{b}y_{b}(\tau)$ denote the remaining energy, with $\tilde{E}_{b}(0)=E_{b}^{\max}>0$. The mask forces $y_b(t)=0$ whenever $\tilde{E}_b(t)<\bar{e}_b$, ensuring $\bar{e}_b y_b(t)\le\tilde{E}_b(t)$. Thus, $\tilde{E}_b(t)\ge0$ for all $t$ by induction, yielding the energy constraint. The same induction applies to the handover allowance: when the remaining allowance is zero, the mask prevents requests that would incur a handover, while otherwise at most one handover can occur per slot. Hence, both constraints hold for every realized trajectory.
\end{IEEEproof}

\subsection{Episode-Averaged Constraint Violation under Multipliers}
\label{subsec:convergence}

We analyze the multiplier updates~\eqref{eq:dual_E}--\eqref{eq:dual_H} for the $I=B+U$ constraints. Let $\bm{\mu}(k)\in\mathbb{R}_{+}^{I}$ and $\bm{C}(k)\in\mathbb{R}^{I}$ denote the stacked multipliers and episode-level violations~\eqref{eq:cost_E}--\eqref{eq:cost_H}, respectively. For the theoretical analysis, we consider a common step-size coefficient $\beta_E=\beta_H=\beta>0$, giving $\bm{\mu}(k+1)=[\bm{\mu}(k)+(\beta/\sqrt{K})\bm{C}(k)]^{+}$ with $\bm{\mu}(1)=\bm{0}$. Define $C_i(\theta)=\mathbb{E}_{\pi_\theta}[C_i(k)]$ as the expected constraint violation and $J(\theta)=\mathbb{E}_{\pi_\theta}[\frac{1}{T}\sum_{t=0}^{T-1}r_t]$ as the expected average return. The corresponding Lagrangian is
\begin{equation}
  L(\theta,\bm{\mu})=J(\theta)-\bm{\mu}^{\top}\bm{C}(\theta),
  \label{eq:lagrangian}
\end{equation}
which equals the expected undiscounted average of the augmented reward~\eqref{eq:reward_form} up to a policy-independent constant for fixed multipliers. We make the following assumptions.

\begin{assumption}[Standing Conditions]
\label{ass:std}
\begin{enumerate}
\item \emph{Slater:} there exist a policy $\pi_{0}$ and $\delta>0$ with $\mathbb{E}_{\pi_{0}}[C_{i}]\le-\delta$ for all $i$.
\item \emph{Bounded costs:} $\|\bm{C}(\theta)\|_{2}\le c\triangleq\sqrt{B\,c_{E}^{2}+U\,c_{H}^{2}}$ for all $\theta$, where $c_{E}$ and $c_{H}$ bound $|C_{E}^{b}|$ and $|C_{H}^{u}|$.
\item \emph{Approximate maximization:} the HAPPO update at episode $k$ satisfies $L(\theta_{k},\bm{\mu}(k))\ge\sup_{\theta}L(\theta,\bm{\mu}(k))-\varepsilon_{k}$ with $\bar{\varepsilon}\triangleq\sup_{k}\varepsilon_{k}<\infty$, and $J(\theta_{k})-J(\pi_{0})\le c_{J}$ for all $k$ and some $c_{J}>0$.
\end{enumerate}
\end{assumption}
Conditions 1) and 2) are standard in primal--dual constrained optimization~\cite{neely2010stochastic}. Condition 3) models the inexact policy optimization performed by HAPPO through a uniformly bounded suboptimality error.

\begin{remark}[Mean Dual Dynamics]
\label{rem:exact_dual}
Theorem~\ref{thm:dual_conv} analyzes the mean dual updates by replacing the sampled constraint violations $\bm{C}(k)$ with their expectations $\bm{C}(\theta_k)$. The implemented algorithm uses realized violations~\eqref{eq:cost_E}--\eqref{eq:cost_H} and discounted advantage estimation with $\gamma<1$ in~\eqref{eq:td_error}, whereas the analysis considers mean dual dynamics and an undiscounted average-return Lagrangian. The theorem does not explicitly characterize the stochastic dual-update noise or the dynamics induced by discounted advantage estimation.
\end{remark}

\begin{lemma}[Dual Boundedness]
\label{lem:dual_bound}
Under Assumption~\ref{ass:std} and the mean dual dynamics of Remark~\ref{rem:exact_dual}, with $\beta_k=\beta/\sqrt{K}$ for $k=1,\ldots,K$,
\begin{equation*}
  \bigl\|\bm{\mu}(k)\bigr\|_{2}\;\le\;\mu_{\max}
  \triangleq\sqrt{\bigl(\bar{\mu}+\beta c\bigr)^{2}+\beta^{2}c^{2}},
  \qquad
  \bar{\mu}\triangleq\frac{c_{J}+\bar{\varepsilon}}{\delta},
\end{equation*}
for all $k=1,\dots,K{+}1$, where $\mu_{\max}$ is independent of $K$.
\end{lemma}
\begin{IEEEproof}
We first show that a large dual vector forces a negative inner product with the violation. By condition 1) and $\mu_{i}(k)\ge0$, the Slater policy satisfies $L(\pi_{0},\bm{\mu}(k))\ge J(\pi_{0})+\delta\|\bm{\mu}(k)\|_{1}$. Since $\sup_{\theta}L(\theta,\bm{\mu}(k))\ge L(\pi_{0},\bm{\mu}(k))$, condition 3) gives $L(\theta_{k},\bm{\mu}(k))\ge J(\pi_{0})+\delta\|\bm{\mu}(k)\|_{1}-\bar{\varepsilon}$, and substituting the definition of $L$ together with $J(\theta_{k})-J(\pi_{0})\le c_{J}$ yields $\delta\|\bm{\mu}(k)\|_{1}\le c_{J}+\bar{\varepsilon}-\bm{\mu}(k)^{\!\top}\bm{C}(\theta_{k})$. Hence, whenever $\bm{\mu}(k)^{\!\top}\bm{C}(\theta_{k})\ge0$ we have $\|\bm{\mu}(k)\|_{2}\le\|\bm{\mu}(k)\|_{1}\le\bar{\mu}$, or equivalently,
\begin{equation}
  \bigl\|\bm{\mu}(k)\bigr\|_{2}>\bar{\mu}
  \;\;\Longrightarrow\;\;
  \bm{\mu}(k)^{\!\top}\bm{C}(\theta_{k})<0.
  \label{eq:contrapositive}
\end{equation}
We now show by induction that $\|\bm{\mu}(k)\|_{2}^{2}\le(\bar{\mu}+\beta_{k}c)^{2}+(k-1)\beta_{k}^{2}c^{2}$ for all $k$. The base case $\bm{\mu}(1)=\bm{0}$ is immediate. Assume the bound holds at episode $k$. If $\|\bm{\mu}(k)\|_{2}\le\bar{\mu}$, the non-expansiveness of $[\cdot]^{+}$ and condition 2) give $\|\bm{\mu}(k{+}1)\|_{2}\le\|\bm{\mu}(k)\|_{2}+\beta_{k}c\le\bar{\mu}+\beta_{k}c$. Otherwise~\eqref{eq:contrapositive} applies, and expanding the update gives $\|\bm{\mu}(k{+}1)\|_{2}^{2}\le\|\bm{\mu}(k)\|_{2}^{2}+2\beta_{k}\bm{\mu}(k)^{\!\top}\bm{C}(\theta_{k})+\beta_{k}^{2}c^{2}\le\|\bm{\mu}(k)\|_{2}^{2}+\beta_{k}^{2}c^{2}$. In either case the bound holds at $k{+}1$. Finally, substituting $\beta_{k}=\beta/\sqrt{K}$ and $k\le K{+}1$ gives $(k-1)\beta_{k}^{2}c^{2}\le\beta^{2}c^{2}$ and $\beta_{k}c\le\beta c$, which is the claim.
\end{IEEEproof}

\begin{theorem}[Average Constraint Violation]
\label{thm:dual_conv}
Under Assumption~\ref{ass:std} and the mean dual dynamics of Remark~\ref{rem:exact_dual}, with a common step-size coefficient $\beta_E=\beta_H=\beta>0$ (so $\beta_k=\beta/\sqrt{K}$ for all $k$), the dual updates~\eqref{eq:dual_E}--\eqref{eq:dual_H} satisfy
\begin{equation}
  \left[\frac{1}{K}\sum_{k=1}^{K}C_{i}(\theta_{k})\right]^{+}
  \!\!\le\;\frac{\mu_{\max}}{\beta\sqrt{K}}
  \;=\;\mathcal{O}\!\left(\frac{1}{\sqrt{K}}\right),
  \qquad\forall i,
  \label{eq:thm_conv}
\end{equation}
with $\mu_{\max}$ the $K$-independent constant of Lemma~\ref{lem:dual_bound}.
\end{theorem}
\begin{IEEEproof}
Since $[z]^{+}\ge z$ for all $z$, the update $\mu_{i}(k{+}1)=[\mu_{i}(k)+\beta_{k}C_{i}(\theta_{k})]^{+}$ implies $C_{i}(\theta_{k})\le\bigl(\mu_{i}(k{+}1)-\mu_{i}(k)\bigr)/\beta_{k}$ for every $k$. Because the step size $\beta_{k}=\beta/\sqrt{K}$ is constant in $k$, summing this inequality over $k=1,\dots,K$ telescopes the right-hand side to $\bigl(\mu_{i}(K{+}1)-\mu_{i}(1)\bigr)\sqrt{K}/\beta$. Using $\mu_{i}(1)=0$ and $\mu_{i}(K{+}1)\le\|\bm{\mu}(K{+}1)\|_{2}\le\mu_{\max}$ from Lemma~\ref{lem:dual_bound} gives $\sum_{k=1}^{K}C_{i}(\theta_{k})\le\sqrt{K}\mu_{\max}/\beta$. Dividing by $K$ and taking the positive part of both sides, which preserves the inequality since the right-hand side is nonnegative, yields~\eqref{eq:thm_conv}.
\end{IEEEproof}

Theorem~\ref{thm:dual_conv} establishes an $\mathcal{O}(1/\sqrt{K})$ bound on the positive part of the constraint violation averaged over training episodes, without requiring convergence of the policy iterates. The multiplier family therefore regulates constraint pressure across episodes, but the bound does not constrain how budgets are consumed within an episode. The same rate extends to the constraint-specific step sizes used in our experiments by analyzing the dual updates in the weighted norm $\|\bm{\mu}\|_{D^{-1}}^2=\bm{\mu}^{\top}D^{-1}\bm{\mu}$, where $D=\mathrm{diag}(\beta_E I_B,\beta_H I_U)$, with modified constants. Section~\ref{subsec:joint} contrasts this inter-episode regulation with the intra-episode behavior of virtual queues.

\subsection{Partial-Horizon Consumption under Queues}
\label{subsec:intra_episode}

Virtual queues start at zero, so their constraint prices build up as excess consumption accumulates. For a greedy per-slot controller such as DDPP, this price buildup can be too slow to prevent early budget concentration. A trajectory-trained policy can instead adjust activation based on the evolving queue state; its ability to regulate partial-horizon consumption is characterized by the following condition.
\begin{assumption}[Queue-Weighted Drift Bound]
\label{ass:queue_weighted}
There exists a constant $\epsilon_{Z}\ge 0$, independent of $\tau$ and of the horizon $T$, such that for all $\tau\in\{0,\dots,T-1\}$ and all $b\in\mathcal{B}$,
\begin{equation}\label{eq:ass5}
  \frac{1}{\tau+1}\sum_{t=0}^{\tau}\mathbb{E}_{\pi_{\theta}}\!\left[
    Z_{b}(t)\bigl(\bar{e}_{b}y_{b}(t)-\bar{E}_{b}\bigr)
  \right]\;\leq\;\epsilon_{Z}.
\end{equation}
\end{assumption}
The term inside the expectation measures queue-weighted excess energy consumption, which becomes positive when activation exceeds the per-slot allowance while the queue is large. Assumption~\ref{ass:queue_weighted} bounds its time average uniformly over partial horizons and horizon lengths. The queue prices in~\eqref{eq:reward_form} penalize excess activation, while queue-state observations allow the policy to respond to accumulated budget pressure. The resulting constant $\epsilon_Z$ depends on the learned policy and is estimated empirically in Section~\ref{subsec:theory_check}.

\begin{proposition}[Partial-Horizon Consumption Bound]
\label{prop:intra_episode}
Under Assumption~\ref{ass:queue_weighted}, for any partial horizon $\tau \in \{0,1,\ldots,T-1\}$ and any BS $b$,
\begin{equation}
  \mathbb{E}\!\left[\sum_{t=0}^{\tau}\bar{e}_{b}\,y_{b}(t)\right]
  \;\leq\;\underbrace{\frac{\tau+1}{T}\,E_{b}^{\max}}_{\text{allocation}}
  \;+\;\underbrace{\sqrt{2\bigl(\epsilon_{Z}+C_{b}\bigr)(\tau+1)}}_{\text{slack}},
  \label{eq:prop2_bound}
\end{equation}
where $C_{b}=\tfrac{1}{2}\max\bigl(\bar{e}_{b}^{2},\bar{E}_{b}^{2}\bigr)$.
\end{proposition}
\begin{IEEEproof}
Let $L_{b}(t)=\tfrac{1}{2}Z_{b}(t)^{2}$ and $a_{b}(t)\triangleq\bar{e}_{b}y_{b}(t)-\bar{E}_{b}$, so that $Z_{b}(t+1)=[Z_{b}(t)+a_{b}(t)]^{+}$. Since $([z]^{+})^{2}\le z^{2}$, the one-step drift satisfies $L_{b}(t+1)-L_{b}(t)\le Z_{b}(t)a_{b}(t)+\tfrac{1}{2}a_{b}(t)^{2}$. As $y_{b}(t)\in\{0,1\}$ gives $a_{b}(t)\in\{-\bar{E}_{b},\,\bar{e}_{b}-\bar{E}_{b}\}$, the second term is bounded by $C_{b}$. Taking expectations, summing from $t=0$ to $\tau$ with $L_{b}(0)=0$, and applying Assumption~\ref{ass:queue_weighted} to the first term gives
\begin{equation}
  \mathbb{E}\bigl[L_{b}(\tau+1)\bigr]
  \;\leq\;\bigl(\epsilon_{Z}+C_{b}\bigr)(\tau+1).
  \label{eq:lyapunov_sum}
\end{equation}
Next, since $Z_{b}(t+1)\ge Z_{b}(t)+a_{b}(t)$ and $Z_{b}(0)=0$, induction on~\eqref{eq:Z} yields $Z_{b}(\tau+1)\ge[\Delta_{b}(\tau)]^{+}$, where $\Delta_{b}(\tau)\triangleq\sum_{t=0}^{\tau}\bar{e}_{b}y_{b}(t)-(\tau+1)\bar{E}_{b}$ is the excess consumption over the proportional allocation. Squaring, taking expectations, and using $\mathbb{E}[Z_{b}(\tau+1)^{2}]=2\,\mathbb{E}[L_{b}(\tau+1)]$ with~\eqref{eq:lyapunov_sum} gives $\mathbb{E}[([\Delta_{b}(\tau)]^{+})^{2}]\le2(\epsilon_{Z}+C_{b})(\tau+1)$, and Jensen's inequality then bounds $\mathbb{E}[[\Delta_{b}(\tau)]^{+}]$ by the square root of the right-hand side. Finally, since $\Delta_{b}(\tau)\le[\Delta_{b}(\tau)]^{+}$, the expected consumption is at most $(\tau+1)\bar{E}_{b}+\mathbb{E}[[\Delta_{b}(\tau)]^{+}]$, and substituting $\bar{E}_{b}=E_{b}^{\max}/T$ yields~\eqref{eq:prop2_bound}.
\end{IEEEproof}

Proposition~\ref{prop:intra_episode} bounds the expected excess consumption over the proportional allocation by $\mathcal{O}(\sqrt{\tau+1})$. Since the allocation grows linearly with $\tau+1$, the relative slack decays as $\mathcal{O}(1/\sqrt{\tau+1})$. Two qualifications apply. The bound is conditional on Assumption~\ref{ass:queue_weighted} and is informative when $\epsilon_Z$ is sufficiently small. It is also one-sided: it bounds excess consumption but provides no lower bound on budget utilization, which is evaluated empirically in Section~\ref{sec:experiments}.

\begin{remark}[Extension to the Handover Budget]
\label{rem:handover_pacing}
The proof uses only the $[\cdot]^{+}$ dynamics and the bounded increment, so it applies verbatim to the handover queue~\eqref{eq:G}: under the analogous condition with constant $\epsilon_{G}$, the expected cumulative handovers of user $u$ satisfy~\eqref{eq:prop2_bound} with $E_{b}^{\max}$, $\epsilon_{Z}$, and $C_{b}$ replaced by $H_{u}^{\max}$, $\epsilon_{G}$, and $\tfrac{1}{2}$.
\end{remark}

\subsection{Queue-Price Buildup and the Drift Constant}
\label{subsec:counterexample}

The queues alone do not determine $\epsilon_Z$; its value also depends on the policy using them. Consider a stylized DDPP controller with $B=1$, $\bar{e}_b=1$, $\bar{E}_b=\eta\in(0,1)$, and a fixed utility gain $G>0$ from activating the BS. With an unscaled queue price, DDPP activates whenever $Z_b(t)<G$. Before the first deactivation, the queue therefore grows as $Z_b(t)=(1-\eta)t$, and activation continues for $t<t^\star=G/(1-\eta)$. During this period, the summand in~\eqref{eq:ass5} equals $(1-\eta)^2t$, giving a running average of $(1-\eta)^2(m-1)/2$ after $m$ consecutive activation slots. Thus, for $m=\min\{\lceil t^\star\rceil,\lfloor\eta T\rfloor\}\ge1$, any admissible drift constant must satisfy $\epsilon_Z\ge(1-\eta)^2(m-1)/2$. When the horizon is long enough for the queue to reach $t^\star$ before budget exhaustion, this lower bound is approximately $(1-\eta)G/2$. Although the queue remains bounded for fixed $G$, a large utility gain therefore requires a large drift constant. During the initial activation period, consumption exceeds the proportional allocation by $(1-\eta)m$ after $m$ slots. Moreover, the slack in~\eqref{eq:prop2_bound} is at least as large as the allocation whenever $\tau+1\le2(\epsilon_Z+C_b)/\eta^2$, illustrating why slow queue-price buildup can make the partial-horizon bound weak.\footnote{For a fixed utility gain and a sufficiently long horizon, the queue oscillates near $G$ and the long-run activation ratio approaches $\eta$. In the experiments, the effective utility gain varies with the fairness queue, so $t^\star$ represents an illustrative price-buildup timescale rather than an exact switching time.}

The distinction lies in how the policy responds to accumulated budget pressure. A greedy controller initially activates until the queue price offsets the instantaneous utility gain, whereas a trajectory-trained policy can learn to reduce activation earlier by conditioning on the evolving queue state. This can lower $\epsilon_Z$, but its effectiveness still depends on the price scale relative to the utility gain, motivating the calibration discussed in Remark~\ref{rem:scale}.

\subsection{Timescale Separation}
\label{subsec:joint}

Per-trajectory feasibility is provided at inference by action masking for every method (Proposition~\ref{prop:feasibility}). During training, the two constraint-handling families regulate budget pressure on different timescales.

The multiplier family operates \emph{across} episodes: the dual variables accumulate constraint pressure from past episodes and remain fixed within each episode. Under the stated assumptions, Theorem~\ref{thm:dual_conv} bounds the episode-averaged expected violation at rate $\mathcal{O}(1/\sqrt{K})$. However, this guarantee does not constrain when resources are consumed within an episode. In the undiscounted objective, fixed multipliers assign the same cumulative penalty to a given total consumption regardless of its temporal distribution. A multiplier-trained policy may still learn to pace consumption through its utility surrogate and remaining-budget observations, but the dual prices themselves do not respond to budget depletion within an episode. The queue family operates \emph{within} each episode: the prices $Z_b(t)$ and $G_u(t)$ respond to accumulated consumption relative to the per-slot allowances, increasing under excess consumption and potentially decreasing when consumption falls below the allowances. Under Assumption~\ref{ass:queue_weighted}, this feedback yields the partial-horizon consumption bound of Proposition~\ref{prop:intra_episode}, with an analogous result for handovers (Remark~\ref{rem:handover_pacing}). The multiplier analysis provides no corresponding partial-horizon bound.

The two timescales entail different tradeoffs. Lagrangian multipliers adapt constraint prices across episodes but require repeated episodes for dual adaptation, whereas virtual queues provide slot-level feedback and exhibit faster training in our experiments. However, queue-based methods can suffer delayed budget pacing because their prices start at zero and must build up. Remark~\ref{rem:scale} explains how price calibration addresses this limitation, which is investigated experimentally in Section~\ref{subsec:robust}.

\begin{remark}[Price Scale]
\label{rem:scale}
The construction of Section~\ref{subsec:counterexample} identifies a key factor in finite-horizon budget pacing: the queue-price buildup time $t^\star$ relative to the available activation allowance $\eta T$. When $t^\star$ approaches or exceeds $\eta T$, a greedy controller may exhaust the budget before the queue penalty becomes large enough to discourage activation. Trajectory-level learning can reduce activation earlier by responding to the observed queue state, but its effectiveness still depends on the relative scales of the utility gain and the queue penalty. In the fixed-gain construction, scaling the energy price to $\alpha_E Z_b(t)$ gives $t^\star\approx G/[\alpha_E(1-\eta)]$, predicting that the price-buildup time decreases inversely with $\alpha_E$. Lagrangian multipliers do not rely on such within-episode price buildup; instead, their prices adapt to constraint violations across training episodes, potentially requiring additional training. Section~\ref{subsec:robust} examines the effect of queue-price calibration under different budget conditions.
\end{remark}

\begin{table}[t]
    \centering
    \caption{Common training hyperparameters and utility-specific settings for the DPP and PF surrogates.}
    \label{tab:helymarl_params}
    \renewcommand{\arraystretch}{1.0}
    \setlength{\tabcolsep}{4pt}
    \footnotesize
    \begin{tabular}{ll|ll}
    \toprule
    \textbf{Parameter} & \textbf{Value} & \textbf{Parameter} & \textbf{Value} \\
    \midrule
    Actor learning rate & $3 \times 10^{-4}$ & Discount factor ($\gamma$) & $0.99$ \\
    Critic learning rate & $1 \times 10^{-3}$ & GAE parameter ($\lambda$) & $0.95$ \\
    Entropy coefficient & $0.05$ & Clip parameter ($\varepsilon_{\rm clip}$) & $0.2$ \\
    Minibatch size & $256$ & Update epochs & $4$ \\
    Hidden size & $128$ & Update interval ($L$) & $128$ \\
    Candidate set size ($N_c$) & $5$ & Value normalization & Enabled \\
    \midrule
    DPP: $V$, $q_0$ & $5$, $0.1$ & PF: $n$, $R_0$ & $50$, $1$~Gbps \\
    PF: $\delta_{\rm PF}$ & $10^{-6}$~Gbps & PF-W: $W$ & $1000$ \\
    \bottomrule
    \end{tabular}
\end{table}

\section{Experiments}
\label{sec:experiments}

We test the theoretical predictions of Section~\ref{sec:theory}, compare the utility--constraint design grid, and examine the effects of budget tightness, horizon length, and deployment size. The results lead to the practical guideline of Section~\ref{subsec:guideline}.

\subsection{Simulation Setup}
\label{subsec:setup}

The default deployment uses $B=3$ BSs arranged in an equilateral triangle within a $100$~m~$\times$~$100$~m area, with inter-site distances of $35$--$50$~m. A common simplified 3GPP-based propagation model~\cite{3gpp36814} uses path loss $\mathrm{PL}(d)=128.1+37.6\log_{10}(d\,[\mathrm{km}])$~dB at $2$~GHz, log-normal shadowing with $4$~dB standard deviation and temporal correlation $0.9$, and no small-scale fading. Each BS transmits at $20$~dBm over $B_{\rm w}=500$~MHz, with a noise density of $-174$~dBm/Hz and a $5$~dB noise figure; rates follow $r_{u,b}(t)=B_{\rm w}\log_2(1+\mathrm{SINR}_{u,b}(t))$. The activation cost is $\bar{e}_b=0.1$~W, and $\hat{\epsilon}_Z$ is reported in units normalized by $\bar{e}_b$. Users move by a boundary-clipped Gaussian random walk with standard deviation $1$~m per axis and slot. Unless stated otherwise, $U=20$, $\eta=0.6$, $\kappa=0.015$, and $T=10^4$.

\begin{table}[t]
\centering
\caption{Empirical drift constants $\hat{\epsilon}_Z(T)$ versus horizon $T$ (normalized to $\bar{e}_b=1$). The last row gives the analytical reference for consecutive activation until budget exhaustion.}
\label{tab:eps}
\renewcommand{\arraystretch}{1.1}
\setlength{\tabcolsep}{5pt}
\footnotesize
\begin{tabular}{lcccc}
\toprule
$T$ & $2000$ & $5000$ & $10000$ & $20000$ \\
\midrule
DPP-Queue  & $18.7$ & $35.1$ & $39.0$ & $32.6$ \\
PF-Queue-W & $15.9$ & $33.0$ & $27.7$ & $33.8$ \\
DDPP       & $95.9$ & $239.9$ & $462.6$ & $623.3$ \\
$(1-\eta)^{2}(\eta T-1)/2$ & $96$ & $240$ & $480$ & $960$ \\
\bottomrule
\end{tabular}
\end{table}

\begin{figure}[t]
    \centering
    \includegraphics[width=1.0\linewidth]{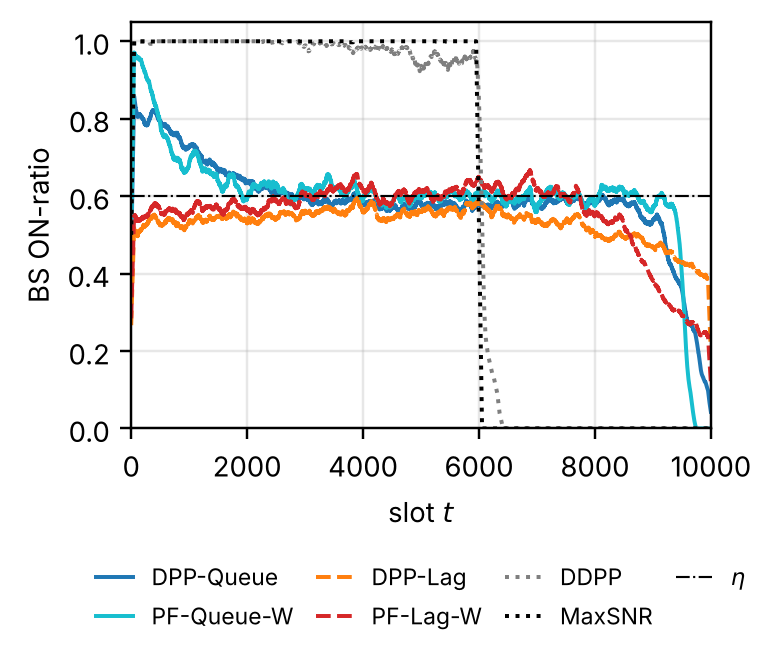}
    \caption{BS ON-ratio trajectories under the default setting, smoothed over $100$ slots, for DPP-Queue, PF-Queue-W, DPP-Lag, PF-Lag-W, and the two non-learned benchmarks.}
    \label{fig:on_ratio}
\end{figure}

All learned methods use the common architecture of Section~\ref{sec:marl}, the hyperparameters of Table~\ref{tab:helymarl_params}, and Adam optimization. The environment, queues, and budget states are reset at each episode. The queue family is trained for $K=10$ episodes. For the multiplier family, the dual step sizes are $(\beta_E,\beta_H)=(1000,10)$ for Jensen-Lag and $(1000,500)$ for DPP-Lag and the PF-Lag variants. Training proceeds in increments of ten episodes until the mean dual variables change by less than $5\%$ over the final five episodes, yielding $K=30$ for Jensen-Lag and $K=70$ for the other multiplier methods. For numerical stability, we set $\delta_Q=10^{-6}$ in the DPP auxiliary-rate denominator and floor the fairness-queue values at $10^{-12}$. Results are averaged over five training seeds, each evaluated on five independent evaluation seeds. Statistical significance is assessed using Welch's $t$-test on training-seed means.

We evaluate throughput as the horizon-average aggregate rate and fairness using Jain's fairness index~(JFI)~\cite{jain1984} of the users' horizon-average rates. Network-wide outage is the fraction of slots with no active BS. We also report the windowed PF utility $\bar{U}_W=\frac{W}{T}\mathbb{E}[\sum_k\sum_u\log(\bar{R}_u^{(k)}+\epsilon)]$, using $W=1000$ as the primary metric and $W=T$ for classical PF. The ON-ratio measures the average fraction of active BSs, while the HO-ratio measures the average handover frequency per user.

Two non-learning benchmarks complete the comparison. \textbf{MaxSNR} associates users with their highest-SNR BSs and schedules users based on instantaneous channel quality, without explicit fairness or budget-price control. \textbf{DDPP} greedily maximizes~\eqref{eq:reward_form} at each slot using the DPP surrogate and virtual-queue prices, without trajectory-level learning. Both benchmarks use the same network model and inference-time budget masking as the learned methods (Remark~\ref{rem:masking}).

\begin{table*}[t]
\centering
\caption{Performance comparison of the design-grid methods and baselines under the default setting ($\eta=0.6$, $\kappa=0.015$, $T=10^4$). $K$ denotes training episodes. Bold indicates the higher $\bar{U}_{1000}$ within each Lagrangian--queue pair; underlining marks the best learned mean in throughput, JFI, outage, and PF utility.}
\label{tab:grid_results}
\renewcommand{\arraystretch}{1.1}
\setlength{\tabcolsep}{5pt}
\footnotesize
\begin{tabular}{llcccccccc}
\toprule
\textbf{Utility} & \textbf{Method} & $K$ & Thr.\ (Gbps) & JFI & ON & HO & Outage & $\bar{U}_{1000}$ & $\bar{U}_{T}$ \\
\midrule
\multirow{2}{*}{--} & MaxSNR & -- & $6.90 \pm 0.25$ & $0.849 \pm 0.029$ & $0.60$ & $0.001$ & $0.400$ & $-70.9 \pm 1.7$ & $-23.3 \pm 1.2$ \\
                    & DDPP   & -- & $5.05 \pm 0.14$ & $0.970 \pm 0.005$ & $0.60$ & $0.003$ & $0.380 \pm 0.012$ & $-59.0 \pm 1.1$ & $-27.7 \pm 0.6$ \\
\midrule
\multirow{2}{*}{Jensen} & Jensen-Lag   & 30 & $6.31 \pm 1.23$ & $0.938 \pm 0.015$ & $0.51 \pm 0.12$ & $0.013$ & $0.144 \pm 0.120$ & $\mathbf{-29.2} \pm 4.1$ & $-24.1 \pm 4.3$ \\
                        & Jensen-Queue & 10 & $6.51 \pm 0.52$ & $0.925 \pm 0.006$ & $0.60$ & $0.006$ & $0.272 \pm 0.041$ & $-50.7 \pm 6.1$ & $-23.3 \pm 1.5$ \\
\midrule
\multirow{2}{*}{DPP} & DPP-Lag   & 70 & $5.92 \pm 0.11$ & $0.975 \pm 0.007$ & $0.53 \pm 0.04$ & $0.013$ & $\underline{0.046} \pm 0.010$ & $-26.9 \pm 1.0$ & $-24.6 \pm 0.4$ \\
                     & DPP-Queue & 10 & $5.85 \pm 0.19$ & $\underline{0.988} \pm 0.003$ & $0.59 \pm 0.02$ & $0.014$ & $0.076 \pm 0.023$ & $\mathbf{-25.9} \pm 1.1$ & $-24.7 \pm 0.6$ \\
\midrule
\multirow{2}{*}{PF} & PF-Lag   & 70 & $7.61 \pm 0.26$ & $0.954 \pm 0.006$ & $0.54 \pm 0.04$ & $0.012$ & $0.088 \pm 0.036$ & $-26.2 \pm 0.8$ & $-19.8 \pm 0.7$ \\
                    & PF-Queue & 10 & $\underline{7.75} \pm 0.65$ & $0.975 \pm 0.002$ & $0.60$ & $0.009$ & $0.104 \pm 0.014$ & $\mathbf{-25.6} \pm 1.1$ & $\underline{-19.3} \pm 1.7$ \\
\midrule
\multirow{2}{*}{PF, $W{=}1000$} & PF-Lag-W   & 70 & $6.71 \pm 0.21$ & $0.976 \pm 0.010$ & $0.56 \pm 0.05$ & $0.012$ & $0.073 \pm 0.021$ & $-25.8 \pm 2.8$ & $-22.1 \pm 0.7$ \\
                                & PF-Queue-W & 10 & $6.83 \pm 0.22$ & $0.986 \pm 0.002$ & $0.60$ & $0.010$ & $0.090 \pm 0.010$ & $\underline{\mathbf{-23.0}} \pm 0.8$ & $-21.6 \pm 0.7$ \\
\bottomrule
\end{tabular}
\end{table*}

\subsection{Empirical Examination of the Theory}
\label{subsec:theory_check}

We test the two timescales of constraint regulation and distinguish their statistical guarantees from the hard feasibility provided by action masking.

\noindent\textbf{Queue-weighted drift.}
Proposition~\ref{prop:intra_episode} gives a useful partial-horizon bound only if the queue-weighted drift is small. Table~\ref{tab:eps} estimates $\hat{\epsilon}_Z(T)$ by maximizing the sample-averaged left-hand side of~\eqref{eq:ass5} over BSs and partial horizons; each learned policy is retrained at its tested $T$. DDPP exhibits nearly linear growth up to $T=10^4$, consistent with budget exhaustion before its queue price binds. In contrast, the learned queue policies have constants between $16$ and $39$ across the tested horizons, more than an order of magnitude smaller than DDPP at $T=10^4$. The empirical handover constants remain below $0.1$. At $T=10^4$, this reduction limits the interval over which the theoretical slack exceeds the proportional allocation to about $220$ slots, versus $2600$ for DDPP. The experiment supports the small-drift regime of Assumption~\ref{ass:queue_weighted} over the measured horizons, without establishing uniformity for arbitrary $T$.

\noindent\textbf{Inter-episode dual adaptation.}
The energy multipliers settle near $40$--$41$ after roughly $25$--$30$ training episodes. Fig.~\ref{fig:on_ratio} shows why adaptation across episodes is different from pacing within one: the multiplier methods maintain nearly flat ON-ratios of $0.53$--$0.54$, whereas the queue methods initially exceed the per-slot allowance and then reduce activation as their prices rise. Both families can distribute service across the horizon, but only virtual queues directly adjust constraint prices in response to within-episode consumption.

\noindent\textbf{Role of action masking.}
We also evaluate trained checkpoints without inference-time masking. At least $96\%$ of trajectories violate an energy or handover budget under every learned method. Thus, the guarantees in Section~\ref{sec:theory} must not be read as hard per-trajectory feasibility; that guarantee is provided by masking (Proposition~\ref{prop:feasibility}).

\subsection{Design-Grid Comparison}
\label{subsec:gridresults}

Table~\ref{tab:grid_results} compares all design-grid methods and the two non-learning benchmarks under the default setting. We first examine the evaluation objective and then compare the effects of utility design and constraint handling.

\noindent\textbf{Windowed versus classical PF.}
We ask whether classical horizon-average PF rewards service throughout a finite horizon. MaxSNR and DDPP exhaust their energy budgets early, leaving approximately $40\%$ of slots without an active BS. Yet MaxSNR reaches $\bar{U}_T=-23.3$, comparable to several learned methods. Under the windowed objective it falls to $\bar{U}_{1000}=-70.9$, while DDPP reaches only $-59.0$. Unlike classical PF, the windowed utility exposes prolonged service interruptions. We therefore use $\bar{U}_{1000}$ for the remaining comparisons.

\noindent\textbf{Utility-surrogate effects.}
Holding the constraint family fixed isolates the effect of the utility surrogate. Replacing DPP with PF increases throughput by $1.7$--$1.9$~Gbps under both families, at a modest cost in JFI. Window-local PF averaging improves $\bar{U}_{1000}$ while sacrificing some throughput; PF-Queue-W achieves the highest mean windowed utility. Jensen-Queue performs poorly because the penalty for unserved users dominates the unscaled queue price, whereas Jensen-Lag adjusts its multiplier to the larger reward scale. Thus, surrogate choice governs the throughput--fairness tradeoff and can alter the price scale required for budget regulation. A separate initialization check also shows the importance of the PF prior: cold-start PF-Lag attains $4.90$~Gbps and $\bar{U}_{1000}=-43.4$, versus $7.61$~Gbps and $-26.2$ with $n=50$ pseudo-samples.

\noindent\textbf{Constraint-handling effects.}
Holding the surrogate fixed isolates the constraint family. For DPP and horizon-wide PF, changes in $\bar{U}_{1000}$ are small ($0.6$--$1.0$ units; $p\ge0.24$). PF-Queue-W leads PF-Lag-W by $2.8$ units, although this difference is not statistically significant ($p=0.11$). The clearer distinction is training cost: Fig.~\ref{fig:learning_curves} shows PF-Queue-W approaching final performance rapidly, while PF-Lag-W adapts across more episodes. Queue methods use $K=10$ episodes versus $K=70$ for the DPP and PF multiplier methods, before accounting for price-scale calibration. The multiplier methods generally leave more energy unused but show less outage, whereas queue methods use more of the allowance and can encounter terminal masking. The next subsection tests whether these tradeoffs persist when the budget or horizon changes.

\begin{figure}[t]
    \centering
    \begin{subfigure}[t]{0.9\linewidth}
        \centering
        \includegraphics[width=\linewidth]
        {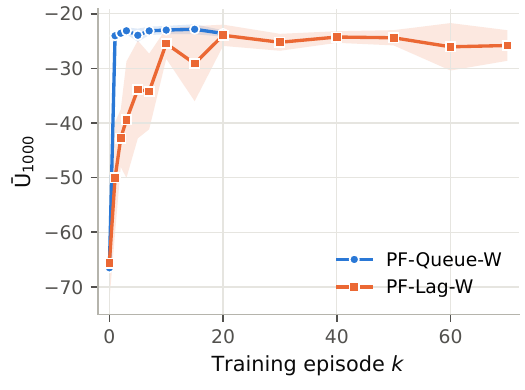}
        \caption{Windowed PF utility $\bar{U}_{1000}$.}
    \end{subfigure}

    \vspace{2pt}

    \begin{subfigure}[t]{0.9\linewidth}
        \centering
        \includegraphics[width=\linewidth]
        {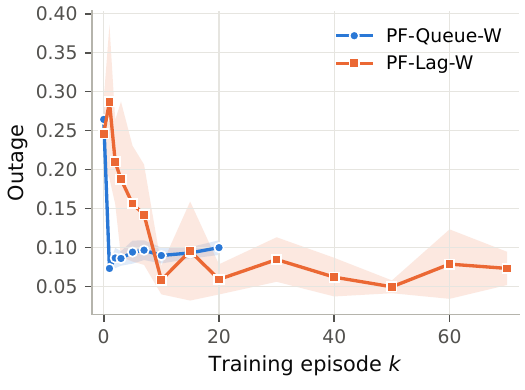}
        \caption{Network-wide outage.}
    \end{subfigure}

    \caption{Learning curves of PF-Queue-W and PF-Lag-W
    under the default setting. Lines and shaded bands
    indicate the mean and standard deviation across
    five training seeds.}
    \label{fig:learning_curves}
\end{figure}

\begin{figure}[t]
    \centering
    \includegraphics[width=0.90\linewidth]{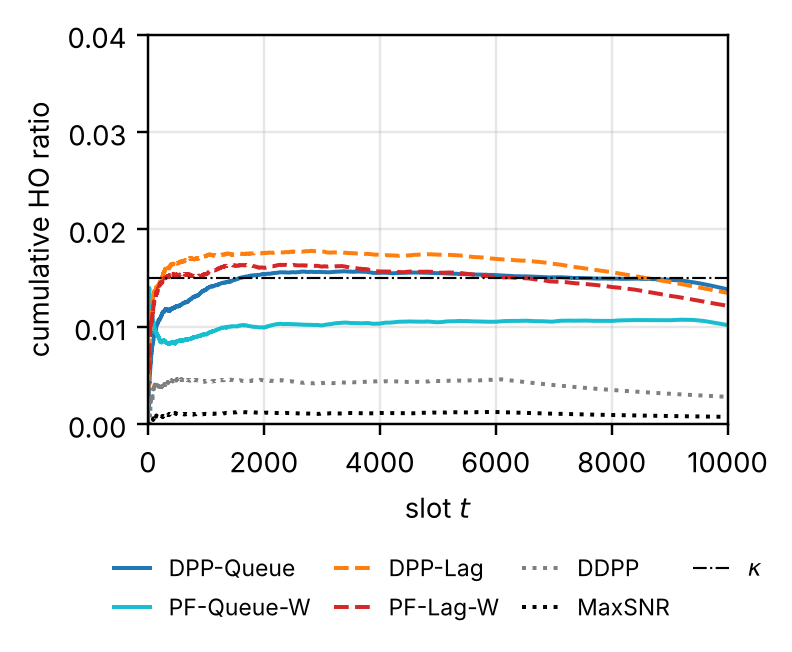}
    \caption{Mean cumulative handover ratios for selected learned methods and the two non-learning benchmarks under the default setting.}
    \label{fig:handover_trajectory}
\end{figure}

\noindent\textbf{Handover pacing.}
We next examine whether the timescale separation also appears in the user-side handover budget. Fig.~\ref{fig:handover_trajectory} shows that DPP-Queue closely tracks the proportional handover allowance after an initial period of approximately $1000$ slots, whereas DPP-Lag consumes handovers faster than the allowance over the first $8000$ slots and approaches $\kappa$ only near the end of the horizon under inference-time masking. This difference illustrates how slot-level handover prices respond to accumulated budget pressure, while episode-level multipliers do not directly enforce temporal pacing. The utility surrogate also affects switching behavior: PF variants generally incur fewer handovers than DPP under either constraint family, consistent with the near-zero average handover multiplier of PF-Lag. MaxSNR provides a low-switching reference based on strongest-SNR association, whereas learned methods use reassociation as an additional means of load balancing under the single-user scheduling restriction. These results indicate that utility design influences handover demand, while constraint-price dynamics govern how the allowance is spent over time.

\subsection{Scalability}
\label{subsec:scalability}

\begin{table}[t]
\centering
\caption{Deployment-size results for DPP-Queue at approximately constant user density ($U/B\approx6.7$), with $\kappa$ adjusted to each deployment's handover demand.}
\label{tab:density}
\renewcommand{\arraystretch}{1.1}
\setlength{\tabcolsep}{2.7pt}
\scriptsize
\begin{tabular}{ccccccc}
\toprule
$B$ & $U$ & $\kappa$ & Thr.\ (Gbps) & Thr./$U$ & JFI & ON \\
\midrule
$3$ & $20$ & $0.015$ & $5.85\pm0.19$  & $0.292$ & $0.988$ & $0.59$ \\
$5$ & $33$ & $0.02$  & $8.75\pm0.33$  & $0.265$ & $0.986$ & $0.57$ \\
$7$ & $47$ & $0.04$  & $11.27\pm0.49$ & $0.240$ & $0.989$ & $0.56$ \\
$9$ & $60$ & $0.07$  & $15.62\pm1.07$ & $0.260$ & $0.983$ & $0.60$ \\
\bottomrule
\end{tabular}
\end{table}

We evaluate DPP-Queue on larger deployments, using a centered square ($B=5$), a centered hexagon ($B=7$), and a $3\times3$ grid ($B=9$). Table~\ref{tab:density} keeps approximately $6.7$ users per BS while adjusting $\kappa$ to each deployment's handover demand. Aggregate throughput increases from $5.85$ to $15.62$~Gbps; per-user throughput remains between $0.24$ and $0.29$~Gbps, JFI exceeds $0.98$, and the ON-ratio stays near the energy allowance. These results extend the evaluation of the shared-policy architecture beyond the default three-BS setting.

\begin{table*}[t]
\centering
\caption{Effect of tighter budgets and shorter horizons, varying one parameter at a time from the default setting. All methods are retrained; queue methods use the unscaled price ($\alpha_E=1$).}
\label{tab:robust}
\renewcommand{\arraystretch}{1.1}
\setlength{\tabcolsep}{6pt}
\footnotesize
\begin{tabular}{llcccc}
\toprule
\textbf{Setting} & \textbf{Method} & $K$ & JFI & Outage & $\bar{U}_{1000}$ \\
\midrule
\multirow{3}{*}{$\kappa=0.01$}
 & DPP-Queue  & 10 & $0.976 \pm 0.008$ & $0.058 \pm 0.031$ & $-27.8 \pm 2.5$ \\
 & PF-Lag-W   & 70 & $0.966 \pm 0.006$ & $0.055 \pm 0.017$ & $-24.7 \pm 1.1$ \\
 & PF-Queue-W & 10 & $0.987 \pm 0.002$ & $0.094 \pm 0.018$ & $-23.1 \pm 0.7$ \\
\midrule
\multirow{4}{*}{$T=2000$}
 & DPP-Lag    & 70  & $0.921 \pm 0.017$ & $0.033 \pm 0.012$ & $-25.1 \pm 1.1$ \\
 & DPP-Queue  & 10  & $0.905 \pm 0.014$ & $0.188 \pm 0.054$ & $-30.6 \pm 2.3$ \\
 & PF-Lag-W   & 100 & $0.917 \pm 0.008$ & $0.054 \pm 0.015$ & $-23.4 \pm 0.6$ \\
 & PF-Queue-W & 10  & $0.912 \pm 0.005$ & $0.164 \pm 0.016$ & $-25.4 \pm 0.9$ \\
\midrule
\multirow{4}{*}{$\eta=0.5$}
 & DPP-Lag    & 70 & $0.982 \pm 0.004$ & $0.063 \pm 0.033$ & $-27.9 \pm 2.1$ \\
 & DPP-Queue  & 10 & $0.986 \pm 0.001$ & $0.233 \pm 0.018$ & $-41.4 \pm 3.8$ \\
 & PF-Lag-W   & 70 & $0.982 \pm 0.008$ & $0.073 \pm 0.018$ & $-24.9 \pm 2.1$ \\
 & PF-Queue-W & 10 & $0.988 \pm 0.001$ & $0.172 \pm 0.008$ & $-28.0 \pm 3.0$ \\
\bottomrule
\end{tabular}
\end{table*}

\subsection{Where the Two Families Separate}
\label{subsec:robust}

The key practical question is whether a queue price can build up before a finite budget is exhausted. Table~\ref{tab:robust} evaluates the unscaled queue methods and the retrained multiplier methods under tighter budgets and a shorter horizon.

\noindent\textbf{Tighter handover budget.}
With $\kappa=0.01$, DPP-Queue's windowed utility decreases from $-25.9$ to $-27.8$, while PF-Queue-W remains near $-23.0$. The more pronounced sensitivity of DPP is consistent with its higher reassociation demand: the two-sided handover constraint can indirectly restrict BS activation as well as switching.

\noindent\textbf{Shorter horizon.}
At $T=2000$, unscaled queue policies exhibit outages of $0.16$--$0.19$, compared with $0.03$--$0.05$ for the multiplier methods. This is consistent with Remark~\ref{rem:scale}: the available activation allowance $\eta T=1200$ is short relative to the queue-price buildup period. The multipliers impose their learned prices from the first slot, although they require more training episodes to obtain those prices.

\noindent\textbf{Lower energy budget.}
At $\eta=0.5$, the gap widens: queue outages reach $0.17$--$0.23$, versus $0.06$--$0.07$ with multipliers. Shrinking the allowance leaves less time for the unscaled queue price to offset utility-driven activation. Together, these tests identify tight budgets and short horizons as conditions under which the unscaled queue price can be insufficient; the next experiment tests whether calibration addresses this limitation. Because the methods use different training budgets, Table~\ref{tab:robust} compares their prescribed training protocols rather than isolating the effect of training duration.

\begin{table}[t]
\centering
\caption{Energy-price calibration for PF-Queue-W ($K=10$), with PF-Lag-W as a reference. Entry$/T$ denotes the first entry of the $500$-slot moving ON-ratio into $\eta\pm0.05$ (minimum $500/T$). The empirical drift $\hat{\epsilon}_Z$ uses unscaled queues, reconstructed from trajectories for PF-Lag-W.}
\label{tab:alpha}
\renewcommand{\arraystretch}{1.1}
\setlength{\tabcolsep}{1.8pt}
\scriptsize
\begin{tabular}{llccccc}
\toprule
\textbf{Setting} & $\alpha_E$ & Outage & $\bar{U}_{1000}$ & ON & Entry$/T$ & $\hat{\epsilon}_Z$ \\
\midrule
\multirow{5}{*}{Default}
 & $1$  & $0.090 \pm 0.010$ & $-23.0 \pm 0.8$ & $0.600$ & $0.16$ & $27.7$ \\
 & $4$  & $0.055 \pm 0.010$ & $-21.7 \pm 0.6$ & $0.600$ & $0.06$ & $6.1$ \\
 & $8$  & $0.053 \pm 0.010$ & $-21.8 \pm 0.7$ & $0.597$ & $0.05$ & $1.6$ \\
 & $16$ & $0.051 \pm 0.016$ & $-21.5 \pm 0.7$ & $0.586$ & $0.05$ & $0.8$ \\
 & Lag-W & $0.073 \pm 0.021$ & $-25.8 \pm 2.8$ & $0.555$ & $0.30$ & $34.1$ \\
\midrule
\multirow{4}{*}{$\eta=0.5$}
 & $1$ & $0.172 \pm 0.008$ & $-28.0 \pm 3.0$ & $0.500$ & $0.19$ & $58.8$ \\
 & $4$ & $0.104 \pm 0.024$ & $-23.2 \pm 1.1$ & $0.500$ & $0.07$ & $12.0$ \\
 & $8$ & $0.084 \pm 0.018$ & $-22.5 \pm 0.9$ & $0.500$ & $0.05$ & $3.7$ \\
 & Lag-W & $0.073 \pm 0.018$ & $-24.9 \pm 2.1$ & $0.435$ & $0.19$ & $58.6$ \\
\midrule
\multirow{4}{*}{$T=2000$}
 & $1$  & $0.164 \pm 0.016$ & $-25.4 \pm 0.9$ & $0.600$ & $0.58$ & $15.9$ \\
 & $4$  & $0.105 \pm 0.016$ & $-24.0 \pm 1.3$ & $0.600$ & $0.35$ & $6.4$ \\
 & $8$  & $0.083 \pm 0.023$ & $-23.7 \pm 1.3$ & $0.600$ & $0.29$ & $3.5$ \\
 & Lag-W & $0.054 \pm 0.015$ & $-23.4 \pm 0.6$ & $0.542$ & $0.46$ & $11.5$ \\
\bottomrule
\end{tabular}
\end{table}

\noindent\textbf{Price-scale calibration.}
We retrain PF-Queue-W with the energy price $\lambda_b^E(t)=\alpha_E Z_b(t)$ to test the price-buildup prediction of Remark~\ref{rem:scale}. Table~\ref{tab:alpha} shows that larger $\alpha_E$ advances entry into the target ON-ratio range and lowers the empirical drift constant: at the default setting, $\hat{\epsilon}_Z$ falls from $27.7$ to $1.6$ as $\alpha_E$ increases from $1$ to $8$. The corresponding estimated partial-horizon slack at mid-horizon decreases to about $5\%$ of the proportional allocation.

Calibration also improves operating performance. At $\eta=0.5$, increasing $\alpha_E$ from $1$ to $8$ reduces outage from $0.172$ to $0.084$ and improves $\bar{U}_{1000}$ from $-28.0$ to $-22.5$. At $T=2000$, outage falls from $0.164$ to $0.083$, remaining numerically above PF-Lag-W's $0.054$ ($p=0.081$). At the default setting, PF-Queue-W with $\alpha_E=8$ exceeds PF-Lag-W in $\bar{U}_{1000}$ by $4.0$ units ($p=0.043$). Gains diminish beyond $\alpha_E=8$, with some energy underutilization at $\alpha_E=16$ under the default setting. These results support price-scale calibration as an effective remedy for premature exhaustion, but do not establish a universal optimal scale; $\alpha_E\in[4,8]$ is a practical range only for the tested conditions.

\subsection{Design Guideline}
\label{subsec:guideline}

The appropriate constraint-handling method depends on the service objective, budget tightness, and available training effort.

\noindent\textbf{Start with virtual queues.}
Their slot-level prices enable effective pacing with fewer training episodes. Begin with the unscaled price and check whether the running ON-ratio approaches the allowance $\eta$ well before budget exhaustion. If early exhaustion persists, increase the queue-price scale. In our experiments, $\alpha_E\in[4,8]$ improves pacing under tight budgets and short horizons, although this range is not universal.

\noindent\textbf{Favor Lagrangian multipliers for short, outage-sensitive horizons.}
Their prices are learned across episodes and apply from the first slot, avoiding within-episode charging delay, at the cost of per-constraint dual step sizes and more training episodes. At the shortest horizon tested, they achieve the lowest observed outage even against calibrated queues. When price calibration is affordable and windowed utility is the priority, calibrated queues remain a competitive alternative.

\noindent\textbf{Align the surrogate with the objective and enforce hard feasibility.}
Use horizon-average PF for long-term throughput and window-local PF when interruptions within shorter windows matter; DPP did not show a statistically significant JFI advantage over window-local PF in our tests. Initialize PF rate averages with a nonzero prior, as supported by the initialization experiment in Section~\ref{subsec:gridresults}, and apply action masking at inference because neither price mechanism alone guarantees per-trajectory feasibility.

\section{Conclusion}
\label{sec:conclusion}

We investigated finite-horizon radio resource management with coupled BS energy and user handover constraints through a unified heterogeneous-agent MARL framework. By comparing three utility surrogates and two constraint-handling mechanisms, we identified their distinct roles in utility optimization and budget regulation. Our theoretical analysis shows that Lagrangian multipliers regulate expected constraint violations across episodes, whereas virtual queues provide conditional partial-horizon consumption bounds. A finite-horizon price-buildup analysis further explains when queue-based regulation becomes ineffective. The experiments support these insights: queues enable faster learning but are sensitive to tight budgets and short horizons, while multipliers adapt constraint prices across training episodes. Importantly, calibrating the queue-price scale substantially improves budget pacing and utility, validating the theoretical explanation. These findings provide practical guidelines for selecting constraint-handling mechanisms according to system requirements, budget tightness, and training resources. Future work will investigate adaptive queue-price calibration and more realistic physical-layer models.

\bibliographystyle{IEEEtran}
\bibliography{refs}

\end{document}